\documentclass[runningheads]{llncs}
\usepackage{graphicx}
\usepackage{tikz}
\usepackage{comment}
\usepackage{amsmath,amssymb} 
\usepackage{color}

\usepackage{booktabs}
\usepackage{algorithm}
\usepackage{algorithmic}
\usepackage{multirow}
\usepackage{pifont}
\usepackage{color}
\usepackage{threeparttable}
\usepackage{url}
\usepackage{xspace}
\usepackage{hyperref}

\usepackage[accsupp]{axessibility}  
\makeatletter
\DeclareRobustCommand\onedot{\futurelet\@let@token\@onedot}
\def\@onedot{\ifx\@let@token.\else.\null\fi\xspace}

\def\eg{\emph{e.g}\onedot} 
\def\ie{\emph{i.e}\onedot}

\def\etal{\emph{et al}\onedot}
\makeatother

\DeclareMathOperator\supp{supp}

\newcommand\norm[1]{\left\lVert#1\right\rVert}
\newtheorem{prop}{Proposition}

\begin{document}
\title{Depth-Aware Image Compositing Model for Parallax Camera Motion Blur}
\titlerunning{Image Compositing Model for Parallax Motion Blur}
\author{German F. Torres \and
Joni Kämäräinen
}
\authorrunning{G. Torres and J. Kämäräinen}

\institute{Tampere University, Finland \\
\email{\{german.torresvanegas, joni.kamarainen\}@tuni.fi} \\
\url{https://github.com/germanftv/ParallaxICB}
}
\maketitle              
\begin{abstract}
Camera motion introduces spatially varying blur due to the depth changes in the 3D world. 
This work investigates scene configurations where such blur is produced under parallax camera motion.
We present a simple, yet accurate, Image Compositing Blur (ICB) model for depth-dependent spatially varying blur.
The (forward) model produces realistic motion blur from a single image, depth map, and camera trajectory.
Furthermore, we utilize the ICB model, combined with a coordinate-based MLP, 
to learn a sharp neural representation from the blurred input.
Experimental results are reported for synthetic and real examples.
The results verify that the ICB forward model is computationally efficient and produces realistic blur,
despite the lack of occlusion information. 
Additionally, our method for restoring a sharp representation proves to be a competitive approach for the deblurring task.

\keywords{blur formation, image compositing blur, neural representations, deblurring}
\end{abstract}
\section{Introduction}
\label{sec:intro}
\begin{figure}[t]
    \centering
    \begin{tabular}{c}
         \includegraphics[width=\linewidth]{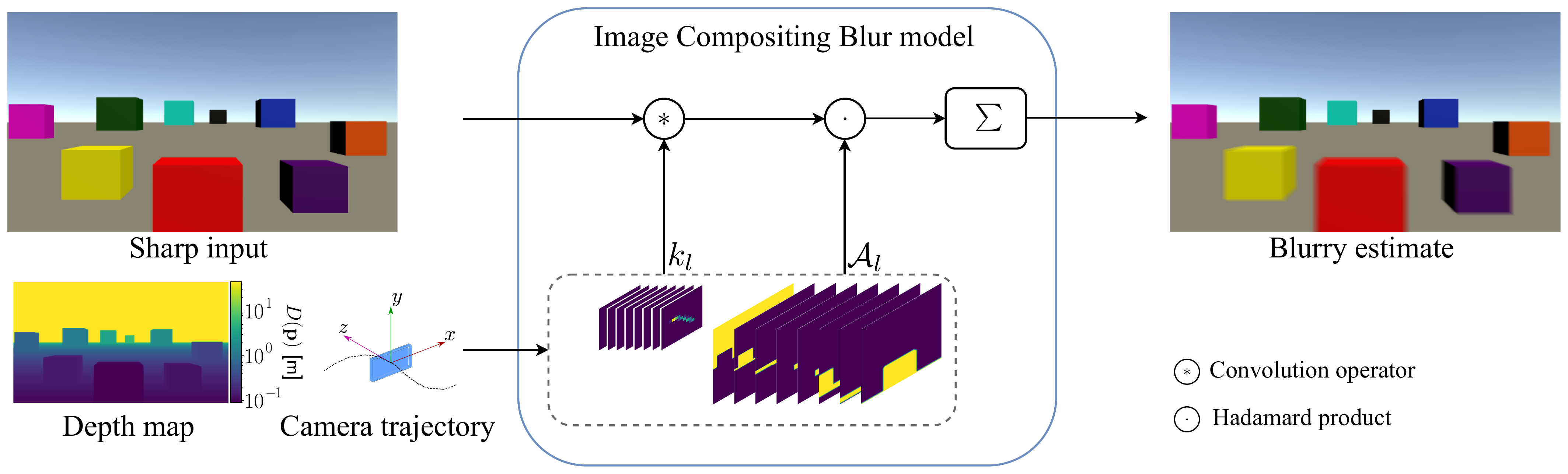}\\
    \end{tabular}
    \caption{
    Our blur formation model accurately describes parallax motion blur. By providing the depth and camera trajectory, we can fit a set of motion kernels $k_l$ and alpha-matte terms $\mathcal{A}_l$, which are used to blend the blur from different layers.
    }
    \label{fig:teaser}
\end{figure}
Motion blur is a common problem in photography and in certain computer vision tasks such as feature matching~\cite{Mustaniemi_2019_WACV} or object detection~\cite{kupyn_CVPR_2018}. 
In essence, motion blur occurs when either the camera or the scene objects, or both, are in motion during exposure.
Recovering the edges and textures of the latent sharp image, \ie deblurring, remains as an open problem since there
are infinite latent sharp sequences consistent with the generated blur.

In conventional deblurring approaches, there is a model that describes the formation of the blur, coupled with an image prior
that regularizes the solution space of the optimization problem.
A major part of the research has been conducted upon suitable image priors that characterize natural images
~\cite{chan_TIP_1998,krishnan_fergus_NIPS_2009,krishnan_CVPR_2011,xu_CVPR_2013,pan_CVPR_2016,li_ICCV_2019}.
However, the applicability in real scenarios also depends on the accuracy of the assumed blur formation model.
Pioneering works assume that the blur results from the shift-invariant convolution of the sharp
image with an unknown Point Spread Function (PSF)
\cite{fergus_SIGGRAPH_2006,levin_CVPR_2009,levin_CVPR_2011,xu_ECCV_2010}.
For this to be precise, either of the two possible scenarios must hold, apart from the scene being static: 
1) the camera shake only involves in-plane translation while the scene is either planar or sufficiently far from the camera, 
2) the focal length of the camera is large and there is no in-plane rotation~\cite{whyte_IJCV_2012}.
Otherwise, camera shake generally induces non-uniform (spatially varying) blur.

Several deblurring algorithms have been proposed to deal with spatially-varying blur that is produced by more realistic 6D camera motion
\cite{tai_PAMI_2010,gupta_ECCV_2010,hirsch_ICCV_2011,whyte_IJCV_2012,xu_CVPR_2013}.
Nevertheless, these works fail at modeling the induced blur in 3D scenes, especially at depth discontinuities.
With the advances in deep learning, several network architectures have been proposed to handle multiple types of blur by learning from data
~\cite{nah_CVPR_2017,kupyn_ICCV_2019,tao_CVPR_2018,ye_IEEEAccess_2020,Zamir_CVPR_2021_MPRNet,Chen_CVPR_2021_HINet,Tsai_TIP_2022_BANet,Cho_ICCV_2021_MIMO-UNet,Zamir_CVPR_2022_Restormer,Tu_CVPR_2022_MAXIM}. 
They benefit from not requiring an explicit description of the blur formation process.
In such works, a neural network is trained over large-scale datasets to restore the sharp image.
Deep deblurring represents state-of-the-art on multiple benchmarks,
but their performance depends on the type of blur that is present in the training set.

Due to the parallax effect, objects positioned at different depths from the camera produce spatially varying blur, 
when the camera moves during capture.
Following this line, a number of works have incorporated the depth in their deblurring methods as an extra auxiliary input~\cite{pan_WACV_2019,sheng_TCSVT_2019,li_TIP_2020} or by a joint estimation process \cite{xu_ICCP_2012,hu_CVPR_2014,zhou_CVPR_2019}, 
but they do not provide a concrete blur model. 

In this work, we study the impact of the depth variation on the motion blur focusing on {\it parallax camera motion}, \ie when the camera moves parallel to the image plane. 
By analyzing the geometry of this type of camera motion, we identify two realistic scene types where
depth plays a significant role in the produced blur: 
{ 1)} {\it Macro} Photography and 
{ 2)} {\it Trucking} Photography.
For such configurations, we propose a tractable Image Compositing Blur (ICB) model for parallax motion
assuming that the depth and camera trajectory are available. 
This model accurately approximates the camera blur under parallax motion (Fig.~\ref{fig:teaser}).
In addition, we provide evidence that our ICB model, in conjunction with coordinate-based Multi-Layer Perceptron (MLP) models, 
can be used to extract a sharp neural representation from a single blurry image.

In summary, the main contributions are:
\textbf{1)} insight analysis about the scene configurations and capture settings for which depth becomes meaningful in the blur formation; 
\textbf{2)} a simplified, yet accurate enough, Image Compositing Blur (ICB) model for parallax depth-induced camera motion blur;
\textbf{3)} an alternative approach to restore sharp images without the need for training over large datasets; and
\textbf{4)} one synthetic and one real dataset of realistic scenes that include pairs of blurry and sharp images with depth maps and camera trajectories.
\section{Related work}
\paragraph{Blur formation models.}
Blur formation models have been studied in the context of image deblurring.
Arguably, the simplest model assumes uniform behavior over the whole image.
In this case, the blurred image is presumed to be the result of shift-invariant convolution with a Point Spread Function (PSF)~\cite{fergus_SIGGRAPH_2006}. 
However, this model only holds for very limited practical scenarios.

For the more general case of spatially-varying blur, some works are based on the projective motion path of the camera shake~\cite{tai_PAMI_2010}. Gupta~\etal.~\cite{gupta_ECCV_2010} assume that the blur can be accurately modeled by in-plane camera translation and rotation. 
White~\etal.~\cite{whyte_IJCV_2012} focus on the blur produced by 3D rotations. 
Furthermore, Hirsch~\etal~\cite{hirsch_ICCV_2011} model the blur as the linear combination of patch-based blur kernel basis. 
Nevertheless, none of these models precisely determine the blur generation in 3D scenes, especially around abrupt changes in depth
\paragraph{Image deblurring.}
Conventional methods for image deblurring are optimization frameworks that tackle the blur produced by the camera motion. 
To handle the well-known ill-posed nature, previous works enforced different image priors in their solutions, 
such as Total-Variation~\cite{chan_TIP_1998},
normalized sparsity prior \cite{krishnan_CVPR_2011},
$L_0$-norm regularization \cite{xu_CVPR_2013},
dark channel prior \cite{pan_CVPR_2016}, 
or discriminative prior \cite{li_ICCV_2019}. 

With the advances in deep learning, several Convolutional Neural Network (CNN) architectures have been proposed.
These architectures only take the blurred image as input and produce the estimated sharp image.
Su~et~al.~\cite{su_CVPR_2017} used an encoder-decoder architecture for video deblurring.
Nah~et~al.~\cite{nah_CVPR_2017} incorporated the multi-scale processing approach in their deep network.
Following the multi-scale principle, numerous CNN-based methods have been introduced 
including components such as Generative Adversarial Networks (GAN)~\cite{kupyn_CVPR_2018,kupyn_ICCV_2019},
Long-Short Term Memory (LSTM)~\cite{tao_CVPR_2018},
scale-iterative upscaling scheme~\cite{ye_IEEEAccess_2020},
half instance normalization~\cite{Chen_CVPR_2021_HINet},
multi-scale inputs and outputs~\cite{Cho_ICCV_2021_MIMO-UNet},
blur-aware attention~\cite{Tsai_TIP_2022_BANet},
and multi-stage progressive restoration~\cite{Zamir_CVPR_2021_MPRNet}.
More recently, progress on Transformer~\cite{vaswani_nips_2017_transformer} and MLP models demonstrate the ability to handle global-local representations for image restoration tasks~\cite{Zamir_CVPR_2022_Restormer,Tu_CVPR_2022_MAXIM}.

The problem with conventional methods is that they do not use depth in deblurring and they are computationally expensive. 
On the contrary, deep deblurring performance strongly depends on the training data which, 
in the case of the above works, do not contain spatial blur induced by depth variations. 
\paragraph{Depth-aware deblurring.}
The involvement of the depth cue in the motion blur, although not widely studied, it is not new. 
Xu and Jia~\cite{xu_ICCP_2012} proposed the first work on this track. 
They used a stereopsis setup to estimate the depth information and subsequently perform layer-wise
deblurring. 
Optimization-based solutions have been introduced for the joint estimation of the scene depth and sharp image,
employing either expectation-maximization~\cite{hu_CVPR_2014} or energy-minimization~\cite{pan_WACV_2019} methods.
Sheng~\etal.~\cite{sheng_TCSVT_2019} proposed an algorithm that iteratively refines the depth and estimates the latent sharp image from an initial depth map and a blurry image, using belief propagation and Richardson-Lucy algorithm, respectively.
Park and Lee~\cite{park_ICCV_2017} proposed and alternating energy-minimization algorithm for the joint dense-depth reconstruction, camera pose estimation, super-resolution, and deblurring.
However, their method requires an image sequence instead of a single image.

On the deep learning side, Zhou~\etal.~\cite{zhou_CVPR_2019} proposed a stereo deblurring network that internally estimates bi-directional disparity maps to convey information about the spatially-varying blur that is caused by the depth variation.
Moreover, Li~\etal.~\cite{li_TIP_2020} introduced a depth-guided network architecture for single-image deblurring, 
which both refines an initial depth map and restores the sharp image.

The above depth-aware deblurring methods properly acknowledge that depth changes produce spatially-varying blur, but it is not clear in which cases this holds. 
The depth is used as an additional cue for deblurring but, on the other hand, they do not address the spatial blur due to scene depth variation.
In contrast, we first identify practical scenarios where the depth variations certainly yield to non-uniform blur. 
We then characterize how depth and camera motion result in regions with different blur behavior.
\section{Geometry of camera motion blur}
\subsection{Fundamentals}
\paragraph{Projective motion path blur model.}
For static scenes, image blur comes from the motion of the camera during the exposure time. 
More precisely, the captured blurry image $\mathbf{y}$ is the summation of the transient sharp images $\{\mathbf{x}_m\}_{m=1}^M$ seen by the camera in the poses $\{\vartheta_m\}_{m=1}^M$ that follow its trajectory. 
Assuming there is a linear transformation $\mathcal{T}_{\vartheta_m}$ that warps the latent sharp image $\mathbf{x}$ to any transient image $\mathbf{x}_m$, the blurred image $\mathbf{y}$ can be expressed as:
\begin{equation}
\label{eq:geometric_motion_blur}
    \mathbf{y} = \sum_{m=1}^M w_m \mathcal{T}_{\vartheta_m}(\mathbf{x}) + \eta \enspace ,
\end{equation}
where the weight $w_m$ indicates the time the camera stays at pose $\vartheta_m$ and $\eta$ error in the model.
The transformation $\mathcal{T}_{\vartheta_m}$ is induced by a homography $H_m$ such that a pixel $\mathbf{p}$ from the latent image $\mathbf{x}$ is mapped to the pixel $\mathbf{p}'_m$ in the transient image $\mathbf{x}_m$.
In homogeneous coordinates, $[ \mathbf{p}'_m ]_h = H_m [ \mathbf{p}]_h $, where $[\cdot]_h$ denotes the conversion from Cartesian to homogeneous coordinates.

For a camera following a 6D motion trajectory, the homography $H_m$ that relates pixels from the latent image $\mathbf{x}$ to the transient image $\mathbf{x}_m$, which are captured from a planar scene at depth $D$, has the form:
\begin{equation}
\label{eq:gupta_homography}
    H_m = C(R_m + \frac{1}{D}T_m[0,0,1])C^{-1} \enspace ,
\end{equation}
where $R_m$ and $T_m$ stand for the rotation and translation components, and $C$ is the intrinsic camera matrix. 
Eq.~(\ref{eq:gupta_homography}) reveals that there is non-uniform blur caused by the depth-dependence of the translation component, 
as well as when rotations are introduced.
Notwithstanding, the homography model only holds for fronto-parallel scenes since the warping operator would require an estimation of the occluded areas that become visible, particularly at the depth discontinuities.
\paragraph{Pixel-Wise Blur (PWB) model.}
In general, image blur has a spatially-varying nature. To take this into account, the blurred image $\mathbf{y}$ can be modeled via 
convolutions with pixel-wise kernels $\mathbf{k}(\mathbf{p},\mathbf{u})$:
\begin{equation}
\label{eq:pixel-wise_kernel_model}
    \mathbf{y}(\mathbf{p}) = \mathbf{x}(\mathbf{p}) * \mathbf{k}(\mathbf{p},\mathbf{u}) + \eta \enspace ,
\end{equation}
where $*$ denotes to the convolution operator, $\mathbf{p}=(i,j)$ are pixel coordinates and $\mathbf{u}=(u,v)$ the kernel coordinates. 
One can blur an image by computing the Empirical Probability Density Function (EPDF) of pixel displacements $\Delta \mathbf{p}'_m = \mathbf{p}'_m  - \mathbf{p}$. 
This model is used as a baseline in our experiments, and its limitations against the proposed blur formation model are demonstrated.

In the remainder of this section, we take a closer look at the influence of depth in the blur generation for in-plane camera translations. Here, we provide insights of what are the scenarios, and to what extent, the depth should be considered in the deblurring problem.
\begin{figure}[t]
    \centering
    \includegraphics[width=0.6\linewidth]{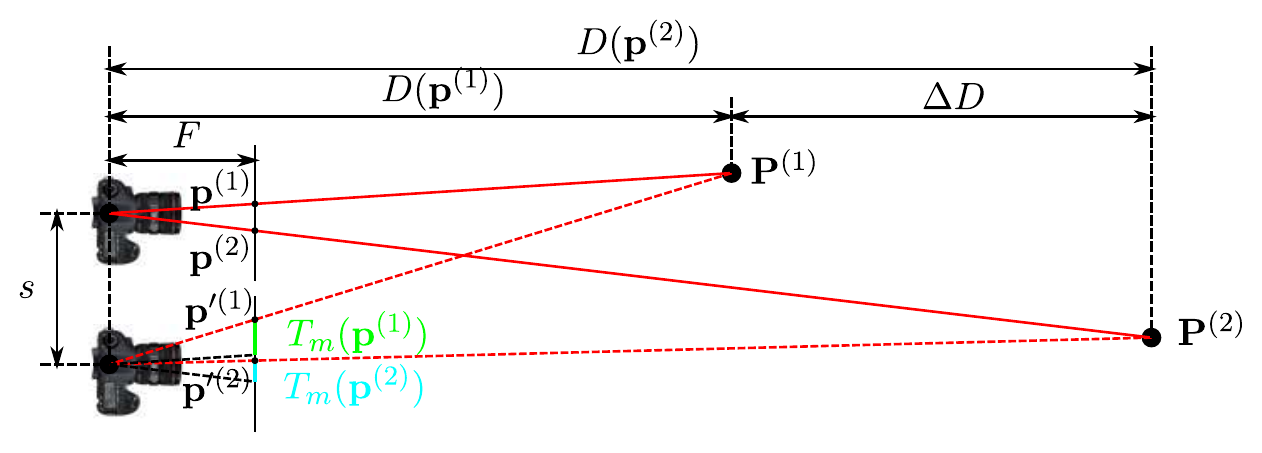}
    
    \caption{Blur induced by camera translation of length $s$ for two 3D points $\mathbf{P}^{(1)}$ and $\mathbf{P}^{(2)}$ with their depth difference of $\Delta D$.}
    \label{fig:inplane_camera_motion_scenario}
\end{figure}
%
\subsection{In-plane camera motion}
\label{sec:inplane_camera_motion_analysis}
\begin{figure*}[t]
    \centering
    \begin{tabular}{cc}
         \includegraphics[width=0.46\linewidth]{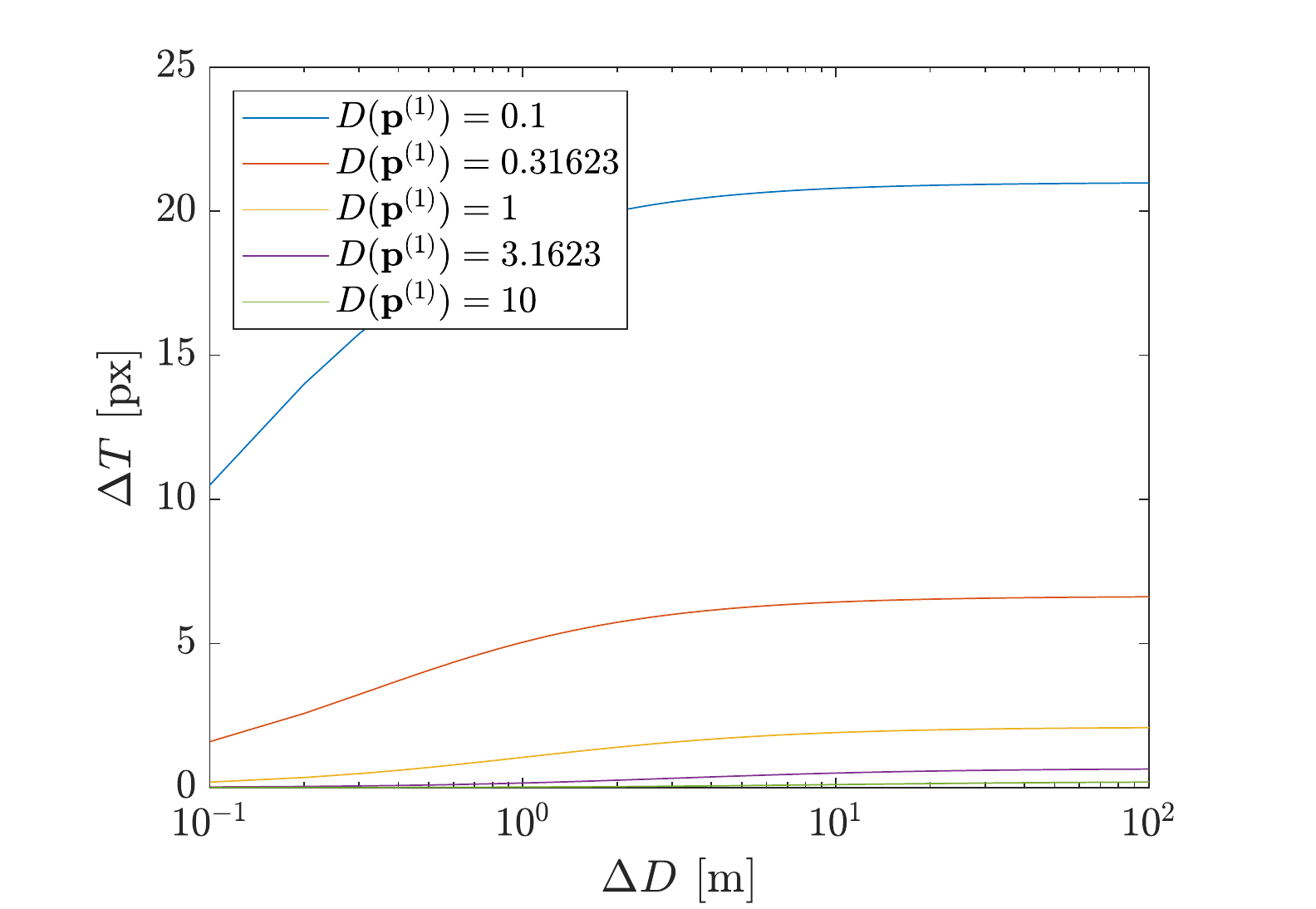}&
         \includegraphics[width=0.46\linewidth]{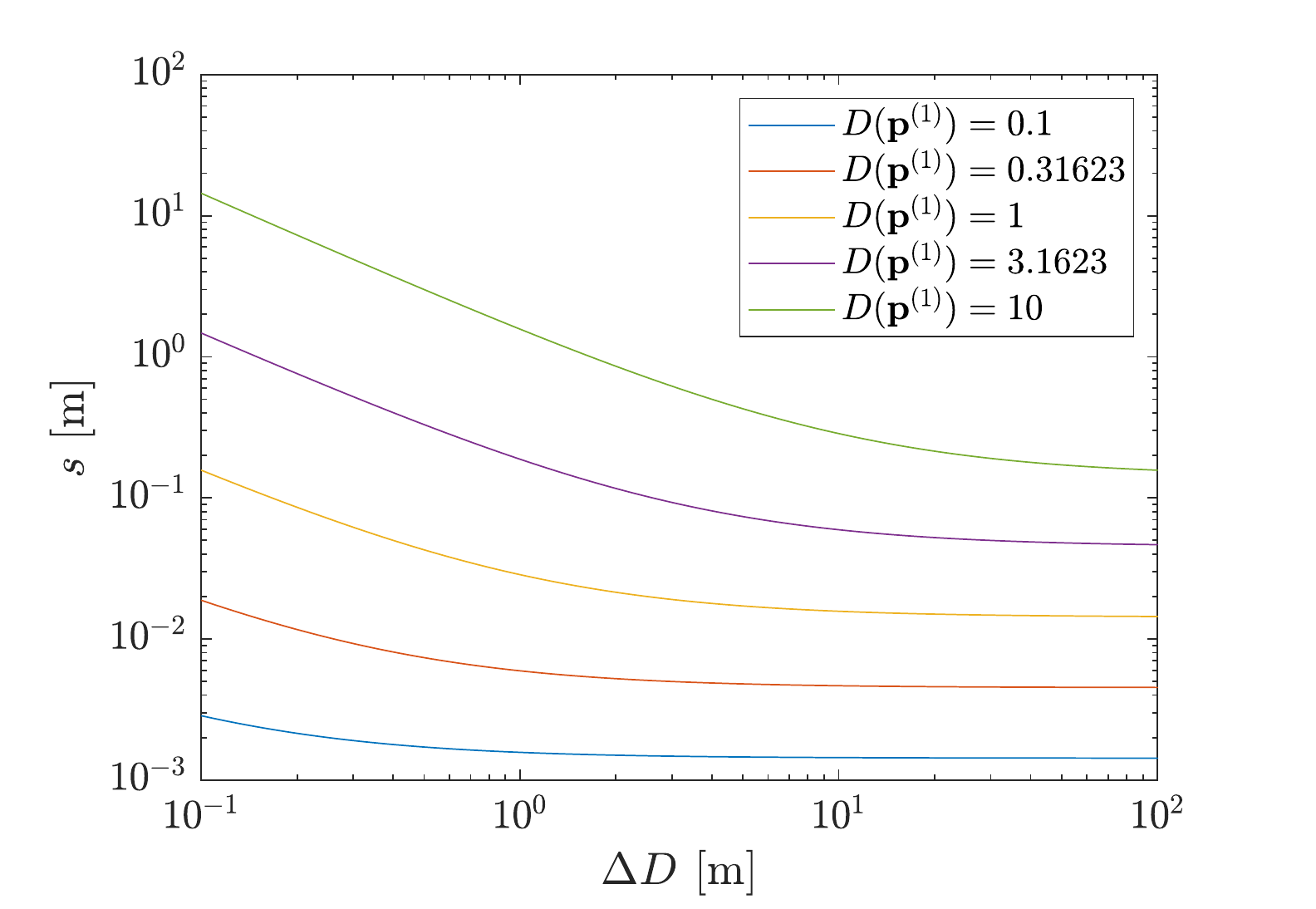} \\
         (a)&(b)
    \end{tabular}
    \caption{\textit{Blur variation} determined by Eq.~\ref{eq:blur_variation}, at different depths of the closest point $D(\mathbf{p}^{(1)})$ (in meters): (a) \textit{blur variation} in pixels as function of the depth difference
    (fixed camera displacement baseline of $s=3[mm]$);
    (b) the camera displacement as a function of the depth difference
    (fixed blur variation $\Delta T = 10$ pixels [px]).
    Camera focal
    length is $F$=2.8[mm] and pixel size 4[$\mu$m] which correspond to settings that can be found in mobile phone cameras (ultrawide lenses).}
    \label{fig:behaviour_blur_variation_inplane_camera_motion}
\end{figure*}
Let us first consider a pin-hole camera with a uniform in-plane motion in the horizontal axis of length $s$ during the exposure time, 
and two trivial 3D points $\mathbf{P}^{(1)}$ and $\mathbf{P}^{(2)}$ such that the former represents the closest point to the camera in the depth direction and the latter is the farthest as depicted in Fig.~\ref{fig:inplane_camera_motion_scenario}. 
On the one hand, $\mathbf{P}^{(1)}$ and $\mathbf{P}^{(2)}$ are respectively mapped to the points $\mathbf{p}^{(1)}$ and $\mathbf{p}^{(2)}$ in the latent image $\mathbf{x}$. 
On the other hand, they are seen, by the camera at pose $\vartheta_m$, on $\mathbf{p}'^{(1)}$ and $\mathbf{p}'^{(2)}$. In this case, the induced homography $H_m(\mathbf{p})$ is given by
\begin{equation}
    H_m(\mathbf{p}) = \begin{bmatrix} 1&0&T_m(\mathbf{p})\\0&1&0\\0&0&1
    \end{bmatrix} \enspace ,
\end{equation}
where $T_m(\mathbf{p})$ is the image plane translation component that is dependent on the pixel depth $D(\mathbf{p})$ as
\begin{equation}
\label{eq:blur_extension_inplane_motion}
    T_m(\mathbf{p})=\frac{sF}{D(\mathbf{p})} \enspace ,
\end{equation}
where $F$ denotes the focal length of the camera. 
Due to the simplicity of the motion, the \textit{blur extent} of an arbitrary 3D point $\mathbf{P}$ in the blurry image $\mathbf{y}$ is given by $T_m(\mathbf{p}) = \mathbf{p}'_x - \mathbf{p}_x$, where $x$ denotes the horizontal component. 
Noteworthy, this is equivalent to the disparity in stereo vision. 
\paragraph{Blur variation.} Since there is a difference in depth $\Delta D = D(\mathbf{p}^{(2)}) - D(\mathbf{p}^{(1)})$,
there must be difference in the blur extent for $\mathbf{P}^{(1)}$ and $\mathbf{P}^{(2)}$, as illustrated in Fig.~\ref{fig:inplane_camera_motion_scenario}. 
Thus, we define the \textit{blur variation} $\Delta T$ as \textit{the difference in blur extent between two points at different depths}. 
Expressively, $\Delta T = T_m(\mathbf{p}^{(1)}) - T_m(\mathbf{p}^{(2)})$. $\Delta T$ measures the non-uniform behavior of the blur caused by the depth and under in-plane camera movements. 
By replacing terms, we get
\begin{equation}
\label{eq:blur_variation}
    \Delta T = \frac{sF}{D(\mathbf{p}^{(1)})\Big[\frac{D(\mathbf{p}^{(1)})}{\Delta D} + 1\Big]} \enspace .
\end{equation}
To gain intuition of the \textit{blur variation} in practical scenarios, we describe its behavior in Fig.\ref{fig:behaviour_blur_variation_inplane_camera_motion} by assuming $F$=2.8[mm] and pixel size of 4[$\mu$m]. 
\paragraph{Macro photography scenes.}
Fig.~\ref{fig:behaviour_blur_variation_inplane_camera_motion}(a) illustrates the \textit{blur variation} $\Delta T$ as a function of the depth difference $\Delta D$, at different depths of the closest point $D(\mathbf{p}^{(1)})$, while keeping a fixed camera displacement $s$=3[mm] (a reasonable choice for natural hand shake). 
It can be seen that whereas the \textit{blur variation} is negligible for far-field scenes no matter what is the depth variation, 
non-uniform blur becomes significant for near-field macro scenes (the closest target $\le$ 0.1m from the camera) even with rather low depth variation ($\ge$ 0.1m). 
Although the \textit{blur variation} increases as the depth difference gets higher, there is an upper bound that is determined by $\Delta T < \frac{sF}{D(\mathbf{p}^{(1)})}$. 
In conclusion, spatially-variant blur is particularly affected by the proximity of the scene whenever there is any variation in depth. 
Consequently, depth plays a significant role for \textit{Macro Photography} scenes.
In this setting, images suffer from defocus blur due to the limited depth-of-field of optics, but defocus blur is a separate issue addressed in other works~\cite{anwar_BMVC_2017,zhang_JVCIR_2016,akpinar_TIP_2021}.
\paragraph{Trucking photography scenes.}
From another perspective, Fig.~\ref{fig:behaviour_blur_variation_inplane_camera_motion}(b) shows the camera displacement $s$ as a function of the depth difference $\Delta D$, 
by assuming a constant blur variation $\Delta T = 10$ pixels, for different depths of the closest point $D(\mathbf{p}^{(1)})$. 
In other words, this plot tells us how much the camera should be moved to produce a \textit{blur variation} of $10$ pixels. 
In this case, it is observed that a few millimeters are sufficient to produce such \textit{blur variation} for near-field scenes, regardless of the depth difference. 
In contrast, in the case of far-field scenes, such a level of blur variation can only be achieved through a camera displacement that ranges from tens of centimeters to a few meters, depending on the depth difference. 
Such intense movement is unlikely to happen in natural hand shake, but appears in cases where the camera is placed on a fast-moving object. 
For example, when capturing pictures from inside a moving car. We dub this as \textit{Trucking Photography} scenes.
\section{Image Compositing Blur (ICB) model}
\label{sec:blur_model}
From Fig.~\ref{fig:behaviour_blur_variation_inplane_camera_motion}(a), we see that there are depth ranges that yield to nearly the same amount of blur.
Hence, pixels in a particular depth range share a common 2D convolutional kernel that characterizes the blur.
Inspired by the defocus blur formation models of Hassinoff~\etal~\cite{hasinoff_ICCV_2007} and Ikoma~\etal~\cite{ikoma_ICCP_2021}, we present a new parallax motion Image Compositing Blur (ICB) model that takes the depth into account:
\begin{equation}
    \mathbf{y} = \sum_{l=0}^{L-1} (\mathbf{x} * k_l)\cdot \mathcal{A}_l + \eta \enspace ,
    \label{eq:image_compositing}
\end{equation}
where $\{\mathcal{A}_l\}_{l=0}^{L-1}$ and $\{k_l\}_{l=0}^{L-1}$ are the set of alpha-matting terms and blur kernels, respectively; and "$\cdot$" is pixel-wise multiplication.
We define each alpha matte as:
\begin{equation}
    \mathcal{A}_l = \frac{\hat{\mathcal{R}}_l \cdot \mathcal{M}_l } {C} \enspace ,
    \label{eq:alpha-matting}
\end{equation}
where $C$ is a normalization constant over the $L$ depth layers (\ie $C := \sum_{l=0}^{L-1} \hat{\mathcal{R}}_l \cdot \mathcal{M}_l $).
$\mathcal{M}_l$ are the z-buffers from far to near layers:
\begin{equation}
    \mathcal{M}_l = \prod_{l'=l+1}^{L-1} (1 - \hat{\mathcal{R}}_l') \enspace .
    \label{eq:cumulative_occlusion}
\end{equation}
$\hat{\mathcal{R}}_l$ is the smooth spatially-extended version of the depth region $\mathcal{R}_l$. $\hat{\mathcal{R}}_l$ is defined as $\hat{\mathcal{R}}_l := (\mathcal{R}_l \oplus \supp{k_l}) * G_{\sigma, \supp{k_l}}$, with $\oplus$ denoting the dilation operator, 
and $G_{\sigma, \supp{k_l}}$ is a Gaussian smoothing window with the standard deviation $\sigma$ and a window size of $\supp{k_l}$. 
$\{\mathcal{R}_l\}_{l=0}^{L-1}$ comes from the discretization of the depth map, 
but dilation and smoothing of $\hat{\mathcal{R}}_l$ are used to approximate the mixed blur around the depth discontinuities, 
and therefore allows to omit explicit estimation of the occluded pixels.
Specifically, $\mathcal{R}_l$ is determined by the scene depth as
\begin{equation}
\label{eq:blur_regions_def}
     \mathcal{R}_l = \begin{cases} 
                        \mathbf{p} \in \Omega |D(\mathbf{p})\geq D_0 &, l=0\\
                        \mathbf{p} \in \Omega | D_{l-1} < D(\mathbf{p}) \leq D_l &, l=1,\dots, L-1 
                    \end{cases} \enspace ,
\end{equation}
where $\Omega$ refers to the pixel domain in the latent image $\mathbf{x}$, 
and $\{D_l\}_{l=0}^{L-1}$ is the sequence of depth values that define the regions with "uniform" blur.
In particular, $D_0$ represents the depth limit value,
the depth values from $D_0$ to $\infty$, for which pixels seem not to move at all.

Next, we derive how to compute the depth sequence $\{D_l\}_{l=0}^{L-1}$ and the respective kernels $k_l$, for the known camera trajectory $s$ and depth map $D(\mathbf{p})$. 

\subsection{Depth-dependent regions}
\label{sec:depth-depentent_regions}
\begin{figure*}[!t]
    \centering
    \includegraphics[width=\linewidth]{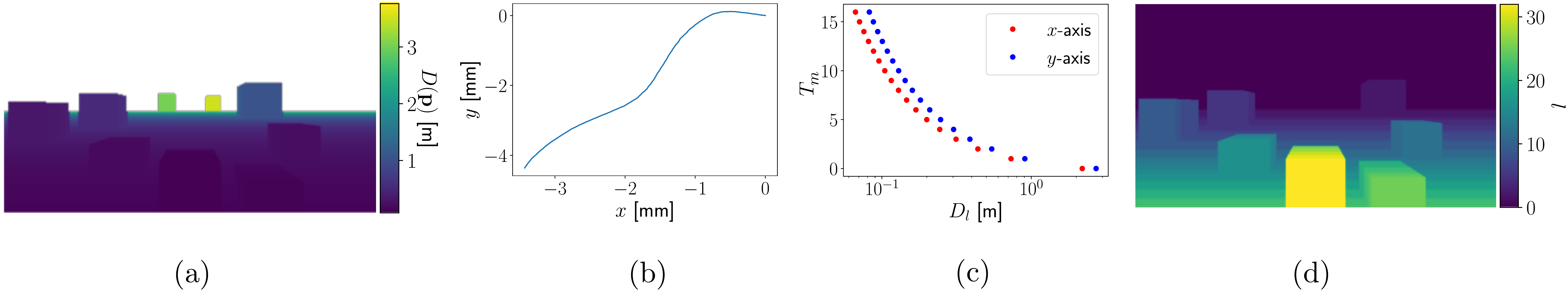}
    \caption{Spatially-varying blur from depth: (a) full depth map of the latent image $\mathbf{x}$, (b) in-plane camera trajectory, (c) depth sequences in the $x$ and $y$ axis that delimit the regions with the same amount of blur (see Eq.~\ref{eq:depth_sequence}) and (d) the region indicators from 0 to 32 that denote the amount of blur from less than 1 to 16 pixels in both dimensions.}
    \label{fig:toy_example_depths}
\end{figure*}
The image regions for which the blur behaves in the same way are completely defined by the depth sequence $\{D_l\}_{l=0}^{L-1}$. Without loss of generality, let us consider a one-dimensional camera movement whose maximum absolute displacement is denoted by $s_{\max}$. As introduced above, we consider $D_0$ the depth limit where pixels do not move, namely those pixels whose blur extent is less than one pixel (half a pixel for rounding issues, in practice). The pixels must satisfy
\begin{equation}
    \frac{\delta}{2} = \frac{s_{\max} F}{D_0} \enspace ,
\end{equation}
where $\delta$ denotes the pixel size. This means that $D_0=2\kappa$ with $\kappa=\frac{s_{\max} F}{\delta}$.

For the rest elements of the sequence, we take into account our definition of \textit{blur variation} presented in Sec.~\ref{sec:inplane_camera_motion_analysis}. The next element in the sequence is characterized as \textit{the depth that produces a blur variation of $n$ pixels~\footnotemark[1]{}}. 
In other words, the blur extent varies $n$ pixels from $\mathcal{R}_{l-1}$ to $\mathcal{R}_l$. This is expressed as:
\begin{equation}
\label{eq:depth_sequence_condition}
    \Delta T = \frac{s_{\max} F}{D_{l}} - \frac{s_{\max} F}{D_{l-1}} = n\delta \enspace .
\end{equation}
\footnotetext[1]{$\sigma$ and $n$ correspond to hyper-parameters in our blur formation model. Ablation studies on those can be found in the supplementary material.}
Similarly, by reorganizing the terms, we find an equation for the $l$-th element of the sequence.
It can be proven by induction that
\begin{equation}
\label{eq:depth_sequence}
    D_l = \frac{2\kappa}{2ln + 1} \enspace ,\hbox{ where } \kappa=\frac{s_{\max} F}{\delta} \enspace .
\end{equation}

To extend this methodology to 2D motion, we simply compute the component-wise sequences $D_{l(x)}$ and $D_{l(y)}$, which are obtained by replacing different values of $s_{\max}$ and $\delta$ for the $x$ and $y$ components of the movement. Then, the complete sequence $D_l$ is the sorted vector of the set union $\{D_{l(x)} \cup D_{l(y)}\}$.
Fig.~\ref{fig:toy_example_depths} exemplifies the discrete regions automatically obtained using the above procedure for a synthetically generated image, with $n=1$. Fig.~\ref{fig:toy_example_depths}(a) and (b) show the full depth and the 2D camera trajectory, respectively. Fig.~\ref{fig:toy_example_depths}(c) illustrates the depth sequences $D_{l(x)}$ and $D_{l(y)}$ computed by using the aforementioned procedure. Lastly, the set of regions $\{\mathcal{R}_l\}_{l=0}^{L-1}$ whose blur behaves similarly within each layer is shown in Fig. \ref{fig:toy_example_depths}(d). 
It is worth mentioning that the total number $L$ of regions is completely adaptive to the scene configuration. One only needs to compute the sequences until $D_{L-1(x)}$ and $D_{L-1(y)}$ cover the minimum depth. 
\subsection{Blur kernels synthesis}
\label{sec:kernels}
Having the time-dependent in-plane camera trajectory $s(t)$ and the depth map $D(\mathbf{p})$, the pixel-wise motion blur kernel is given by
\begin{equation}
    \label{eq:pixel-wise_true_kernel}
    \mathbf{k}(\mathbf{p}) = \hat{f}\Big(\Big\lfloor\frac{-s(t)F}{\delta D(\mathbf{p})}\Big\rfloor\Big) \enspace ,
\end{equation}
where $\lfloor \cdot \rfloor$ denotes the rounding operation and $\hat{f}$ is the operator that computes the EPDF for the discretized values in the argument. Instead of computing kernels $\mathbf{k}(\mathbf{p})$ at pixel level, we compute a smaller set $\{k_l\}_{l=0}^{L-1}$ where every kernel is paired with a region $\hat{\mathcal{R}}_l$.
The pixels in the region $\hat{\mathcal{R}}_l$ share the same motion blur kernel $k_l$:
\begin{equation}
    \label{eq:kernels_estimation}
    k_l = \hat{f}\Big(\Big\lfloor\frac{-s(t)F}{\delta D_l^*}\Big\rfloor\Big) \enspace ,
\end{equation}
where $D_l^*$ is the optimal depth value in the range $[ D_l, D_{l-1} ]$ that minimizes the mean-square error in $D(\mathbf{p})$ for $\mathbf{p} \in \mathcal{R}_l$. The depths $D(\mathbf{p})$ in $[D_l, D_{l-1}]$ follow a random variable $\zeta$ with PDF $f(\zeta)$ and the mean-square error is determined by
\begin{equation}
\label{eq:mse_depth_range}
    \int_{D_l}^{D_{l-1}} (\zeta - D_l^*)^2f(\zeta)d\zeta
    \enspace .
\end{equation}
It can be proven that the mean depth $\Bar{D}(\mathbf{p})$ in the range minimizes (\ref{eq:mse_depth_range}).
\section{Neural representations from blur}
Advances in implicit neural representations demonstrate that MLPs can learn the high-frequency details in 2D images~\cite{tancik_NIPS_2020_fourier,sitzmann_NIPS_2020_siren}.
In those works, a coordinate-based MLP $\Phi_\theta$ optimizes its parameters $\theta$ to fit a sharp image, \ie $\Phi_\theta: \mathbf{p} \mapsto \mathbf{x}$.
We propose a different approach where $\Phi_\theta$ fits the sharp image $\mathbf{x}$ from its corresponding blurred one $\mathbf{y}$, 
by embedding a blur function $b: \mathbf{x}\mapsto \mathbf{y}$ defined by either the PWB model~(\ref{eq:pixel-wise_kernel_model}) or our ICB model~(\ref{eq:image_compositing}).
This provides an alternative solution for deblurring from a single blurred image.
Since $b$ is differentiable, we can use gradient-descent methods to optimize $\theta$ with the following loss:
\begin{equation}
    \mathcal{L}= \sum_{\mathbf{p}}\norm{b(\Phi_\theta(\mathbf{p})) - \mathbf{y}(\mathbf{p})}_{2}^{2} + \lambda \norm{\nabla_\mathbf{p} \Phi_\theta(\mathbf{p})}_1^1 \enspace ,
\end{equation}
where $\lambda$ is a hyper-parameter that controls the smoothness of the gradients.
This method is similar to the approach presented by Ulyanov~\etal~\cite{ulyanov_2018_DIP}, with the exception that we utilize a coordinate-based MLP rather than a CNN for fitting $\mathbf{x}$.
In practice, we use the SIREN architecture~\cite{sitzmann_NIPS_2020_siren} for its ability to fit derivatives robustly.
\section{Experiments}
\subsection{Evaluation Datasets}
\paragraph{Synthetic dataset.} 
We constructed the Virtual Camera Motion Blur (VirtualCMB) dataset, 
where the ground-truth latent images and depth maps are rendered from the 3D scene models. 
We utilized the Unity engine~\cite{unity_2021} for rendering 3D scenes in HD resolution. The dataset was built using five high-quality scenes available in the unity asset store. 
The viewpoints were manually selected to represent a virtual snapshot camera and motion blur for the three studied cases:
\textbf{1)} {\em Macro Photography},
\textbf{2)} {\em Trucking Photography} and
\textbf{3)} {\em Standard Photography}. 
Table \ref{tab:captured_images} summarizes the number of images captured for each case. 
Macro and Trucking represent practical settings where depth contributes to blurring (see Sec.~\ref{sec:inplane_camera_motion_analysis}). 
Standard Photography is the typical setting where all scene objects are far from the camera and thus depth-agnostic models work well. 
In all cases, the camera was moved through pre-defined trajectories. For the Macro and Standard cases, we randomly selected six trajectories from the Kohler dataset~\cite{kohler_ECCV_2012}.
For the Trucking Photography cases, six linear trajectories with a constant speed in the $xy$ plane were generated.
The purpose is to mimic photography from a moving object (e.g., inside a car). To test our method beyond motion parallax,
also camera motions of \textit{pan-tilt} rotations and full 6-DoF camera motion were recorded.
Overall, 983 blurred images with corresponding latent sharp images and depth maps were rendered.
\begin{table}[t]
\centering
\caption{Summary of captured images: i) parallax motion, ii) with \textit{pan-tilt} rotations (w/ $xy$ rotations), and iii) 6-DoF.}
\resizebox{0.52\columnwidth}{!}{%
\begin{tabular}{lccccccccc}
\toprule
\multicolumn{1}{c}{\multirow{2}{*}{Scene}} & \multicolumn{3}{c}{Macro}                 & \multicolumn{3}{c}{Trucking}                 & \multicolumn{3}{c}{Standard}         \\ \cline{2-10} 
\multicolumn{1}{c}{}                       & i)           & ii)          & iii)         & i)           & ii)          & iii)         & i)         & ii)        & iii)         \\ \midrule
VikingVillage                              & 26           & 28           & 26           & 23           & 22           & 23           & -          & -          & 30           \\ 
IndustrialSet                              & -            & -            & -            & 60           & 60           & 60           & -          & -          & 30           \\ 
ModularCity                                & -            & -            & -            & 60           & 60           & 60           & -          & -          & 30           \\ 
ModernStudio                               & 58           & 55           & 55           & -            & -            & -            & -          & -          & 30           \\ 
LoftOffice                                 & 56           & 50           & 51           & -            & -            & -            & -          & -          & 30           \\ \midrule
\textbf{Total: 983}                             & \textbf{140} & \textbf{133} & \textbf{132} & \textbf{143} & \textbf{142} & \textbf{143} & \textbf{-} & \textbf{-} & \textbf{150} \\ \bottomrule
\end{tabular}
}
\label{tab:captured_images}
\end{table}
\paragraph{Real dataset.} 
For evaluation with real images, we used the iOS app introduced by Chugunov~\etal~\cite{chugunov_CVPR_2022} to capture synchronized RGB, LiDAR depth maps, and camera poses.
These videos match with the {\em Macro photography} case, where a static object is recorded by a hand-held smartphone camera.
As preprocessing, RGB frames are down-scaled to the depth map resolution (256x192), blurry frames are obtained by temporal average, while the sharp and depth correspondences are taken from the middle point in the camera trajectory.
Accordingly, we built the Real Camera Motion Blur (RealCMB) dataset, comprised of 58 pairs of blurry and sharp images, as well as depth and camera motion; from which 48 come from our own recordings and 10 are available in \cite{chugunov_CVPR_2022}.
%
\subsection{Model validation}
\paragraph{Parallax motion blur.}
Our ICB model
in Sec.~\ref{sec:blur_model} was particularly designed for the parallax motion of a camera. Thus, we first evaluate the model under in-plane camera motion in the VirtualCMB dataset.
For comparison, we considered the PWB model~(\ref{eq:pixel-wise_kernel_model}).
The standard image quality metrics: PSNR and SSIM, and the perceptual quality metric LPIPS~\cite{zhang_CVPR_2018}, are used for performance evaluation.

Parallax results are in the "Parallax" column of Table.~\ref{tab:model_val_VirtualCMB}.
In the terms of PSNR, SSIM, and LPIPS, the proposed ICB model outperforms the baseline PWB model in both of the main cases: Macro and Trucking, except for the LPIPS in the Trucking case.
Although the difference between SSIM and LPIPS is marginal in practice.
By nature, PWB cannot properly trace the generated blur over the depth discontinuities where occluded areas become visible during the motion.  
Conversely, our ICB model merges blur from different depth layers more effectively, resulting in a more realistic blur.
Fig.~\ref{fig:model_val}(a) and (b) illustrates this finding for the two types of blur, Macro, and Trucking. The error images reveal that the proposed model is more precise at the object edges.
\paragraph{Out-of-plane rotations and 6-DoF.}
\label{sec:results_beyond_parallax}
Non-uniform blur does not only come from motion parallax but also rotations.
Assuming a large focal length as in~\cite{whyte_IJCV_2012}, \textit{pan-tilt} rotations can be approximated by $xy$ translations that are non-depth dependent. Thus, we can compute a global uniform kernel which is added on top of the kernels in Sec.~\ref{sec:kernels}.
In particular, this approximation works well in narrow-lens devices.
This approach was adopted to handle motion camera blur beyond motion parallax, neglecting the effect of $z$ translation and \textit{roll} rotation. 
The results beyond parallax motion (xy-rotation and 6-DoF) in Table~\ref{tab:model_val_VirtualCMB} are similar to the parallax motion experiment. 
Consequently, the used approximation works well for the captured images in the VirtualCMB dataset.
\begin{table*}[!t]
\centering
\caption{Blur formation results in VirtualCMB.} 
\resizebox{\linewidth}{!}{
\begin{tabular}
{l|cc|cc|cc|cc|cc|cc|cc}
\toprule
\multirow{3}{*}{} & \multicolumn{4}{c|}{Parallax} & \multicolumn{4}{c|}{w/ $xy$ rotation} & \multicolumn{6}{c}{6 DoF} \\ 
 & \multicolumn{2}{c|}{Macro} & \multicolumn{2}{c|}{Trucking} & \multicolumn{2}{c|}{Macro} & \multicolumn{2}{c|}{Trucking} & \multicolumn{2}{c|}{Macro} & \multicolumn{2}{c|}{Trucking} & \multicolumn{2}{c}{Standard} \\ 
 & \multicolumn{1}{c}{Ours} & \multicolumn{1}{c|}{PWB} & \multicolumn{1}{c}{Ours} & \multicolumn{1}{c|}{PWB} & \multicolumn{1}{c}{Ours} & \multicolumn{1}{c|}{PWB} & \multicolumn{1}{c}{Ours} & \multicolumn{1}{c|}{PWB} & \multicolumn{1}{c}{Ours} & \multicolumn{1}{c|}{PWB} & \multicolumn{1}{c}{Ours} & \multicolumn{1}{c|}{PWB} & \multicolumn{1}{c}{Ours} & \multicolumn{1}{c}{PWB} \\ 
\midrule
$\uparrow$PSNR                     & {\bf 42.48} & 41.54 & {\bf 37.42} & 36.59 & {\bf 42.16} & 41.11 & {\bf 36.99} & 36.21 & {\bf 38.84} & 38.37 & {\bf 37.08} & 36.29 & {\bf 37.38} & 36.97 \\
$\uparrow$SSIM                     & {\bf 0.993} & 0.992 & {\bf 0.985} & 0.984 & {\bf 0.990} & 0.989 & {\bf 0.984} & 0.983 & {\bf 0.984} & 0.982 & {\bf 0.984} & 0.983 & {\bf 0.978} & 0.978 \\
$\downarrow$LPIPS ($\times 10^{-5}$) & {\bf 4.639} & 5.049 & 5.870 & {\bf 5.238} & {\bf 5.399} & 6.316 & 6.158 & {\bf 5.676} & {\bf 8.883} & 9.744 & 6.088 & {\bf 5.613} & {\bf 11.17} & 13.01 \\
\bottomrule
\end{tabular}
}
\label{tab:model_val_VirtualCMB}
\end{table*}
\begin{table}[!t]
\centering
\caption{Blur formation results for RealCMB (avg over 58 test images).}
\resizebox{0.35\linewidth}{!}{
\begin{tabular}{lccc}
\toprule
            &           $\uparrow$PSNR    &       $\uparrow$SSIM    &       $\downarrow$LPIPS ($\times 10^{-5}$) \\ \midrule
PWB         &           36.31  &       0.984   &        6.826   \\ 
Ours        &           {\bf 38.21}  &       {\bf 0.990}   &        {\bf 4.484}   \\ \bottomrule
\end{tabular}
}
\label{tab:model_val_HNDR}
\end{table}
\begin{table}[!t]
\centering
\caption{Results in terms of computational resources.}
\resizebox{0.55\linewidth}{!}{
\begin{tabular}{l|cc|cc}
\toprule
 & \multicolumn{2}{c|}{RealCMB} & \multicolumn{2}{c}{VirtualCMB} \\
 & Memory {[}MB{]} & Run time {[}s{]} & Memory {[}MB{]} & Run time {[}s{]} \\
 \midrule
PWB & 57.39 & {\bf 1.95} & 2799 & 13.41 \\
Ours & {\bf 1.74} &  2.78 & {\bf 48.97} & {\bf 7.27}\\
\bottomrule
\end{tabular}
}
\label{tab:comp_resources}
\end{table}
\begin{figure*}[h]
    \centering
    \begin{tabular}{c}
         \includegraphics[width=0.94\linewidth]{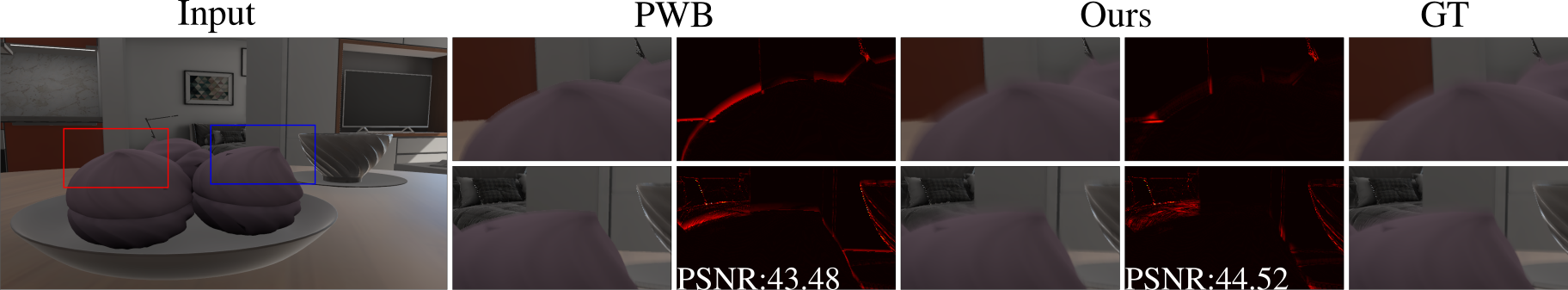} \\

         (a) \\
         \includegraphics[width=0.94\linewidth]{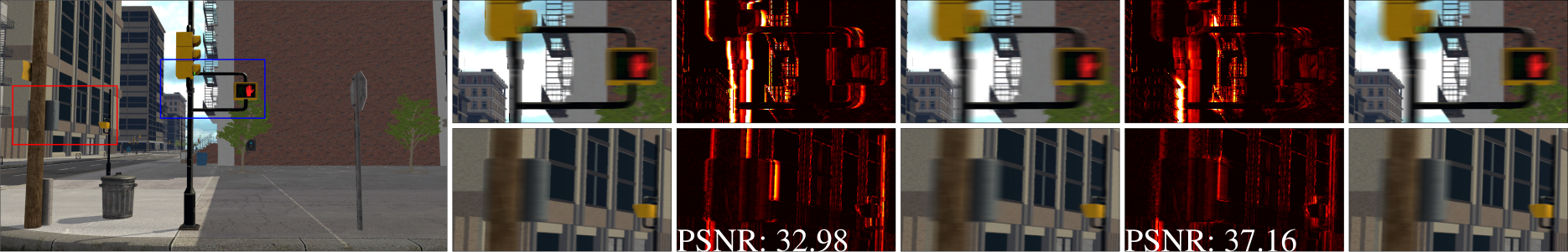} \\

         (b) \\
         \includegraphics[width=0.94\linewidth]{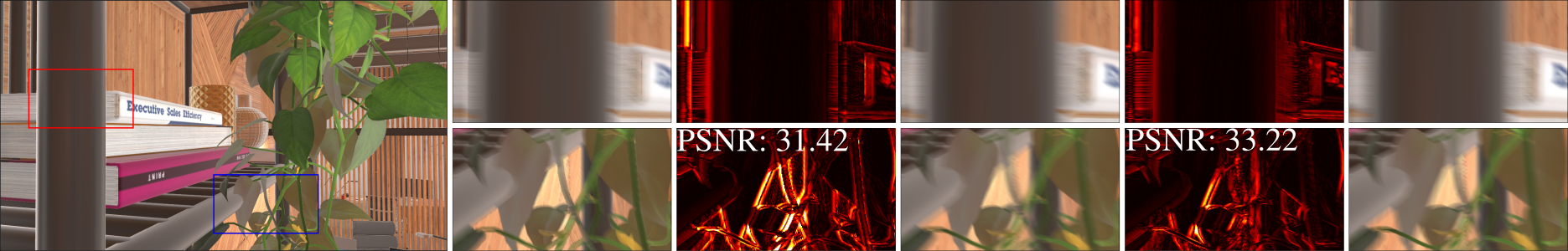} \\

         (c) \\
         \includegraphics[width=0.94\linewidth]{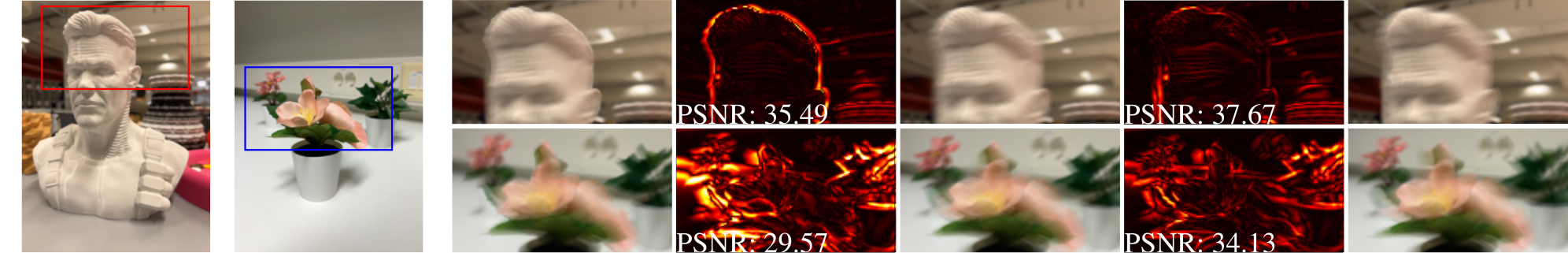} \\
         (d) \\
    \end{tabular}
    \caption{Examples of the blur formation results. (a) Macro, (b) Trucking (c) Macro with 6-DoF motion, and (d) Real images.}
    \label{fig:model_val}
\end{figure*}
%
\paragraph{Real images.} %
Surprisingly, the proposed ICB model performs clearly better on real 6-DoF motion in the RealCMB dataset (Table~\ref{tab:model_val_HNDR}) than in the previous experiment with synthetic data.
These results indicate that 1) parallax motion can be more dominating in real data than in our simulated cases and 2) our model is robust to depth and trajectory noise that appears in real data. 
Moreover, as the depth measurements are non-linearly quantized in ICB (see Sec.~\ref{sec:depth-depentent_regions}), there is no need for high-resolution depth maps.
Fig.~\ref{fig:model_val}(d) shows a visual example of the blur generation in the RealCMB dataset.
%
\paragraph{Computational resources.}
Table~\ref{tab:comp_resources} reports the averaged run time and memory size of our Pytorch implementations of PWB and ICB.
It turns out that the proposed ICB model is slightly slower in the RealCMB dataset but significantly faster in VirtualCMB.
Most importantly, ICB demonstrates considerably greater efficiency in terms of memory consumption, 
with reductions of $\times$32 and $\times$50 in the RealCMB and VirtualCMB datasets, respectively.
\subsection{Neural representations from blur}
\begin{table}[!t]
\centering
\caption{Comparison of sharp restoration results.}
\resizebox{0.9\linewidth}{!}{
\begin{tabular}{lcccccc}
\toprule
  & \multicolumn{3}{c}{VirtualCMB}                         &  \multicolumn{3}{c}{RealCMB}                           \\ 
                   & $\uparrow$PSNR & $\uparrow$SSIM & $\downarrow$LPIPS ($\times 10^{-4}$) & $\uparrow$PSNR & $\uparrow$SSIM & $\downarrow$LPIPS ($\times 10^{-4}$) \\ \midrule
SRN~\cite{tao_CVPR_2018}        &                   29.92 & 0.9135 & 7.423 &       28.26 & 0.9146 & 8.492     \\
SIUN~\cite{ye_IEEEAccess_2020}       &              29.75 & 0.9114 & {\bf 7.235} &       28.33 & 0.9139 & 7.231      \\
HINet~\cite{Chen_CVPR_2021_HINet}      &            29.86 & 0.9133 & 8.651 &       28.32 & 0.9133 & 7.627      \\
BANet~\cite{Tsai_TIP_2022_BANet}      &             29.77 & 0.9099 & 9.007 &       28.34 & 0.9140 & 8.241      \\
MIMO-UNet++~\cite{Cho_ICCV_2021_MIMO-UNet}&         28.79 & 0.8964 & 13.12 &       27.92 & 0.9106 & 10.33      \\
MPRNet~\cite{Zamir_CVPR_2021_MPRNet}     &          30.01 & 0.9146 & 8.062 &       28.79 & 0.9178 & 7.769     \\
MAXIM~\cite{Tu_CVPR_2022_MAXIM}      &              30.34 & {\bf 0.9186} & 7.418 &       28.89 & 0.9207 & 5.695      \\
Restormer~\cite{Zamir_CVPR_2022_Restormer}  &       {\bf 30.41} & 0.9174 & 7.318 &       29.56 & 0.9243 & 6.887      \\
\midrule
PWB + SIREN &                                       27.08 & 0.7975 & 37.25 &        30.61 & 0.9429 & 5.398      \\
Ours + SIREN &                                      27.14 & 0.8001 & 36.56 &        {\bf 31.92} & {\bf 0.9546} & {\bf 4.032}      \\

\bottomrule
\end{tabular}
}
\label{tab:deblurring_results}
\end{table}
Table~\ref{tab:deblurring_results} summarizes the results of the sharp implicit representations with ICB and PWB models.
In addition, we evaluated SOTA deep-deblurring methods.
For the task of learning implicit representations from a single blurry image, our ICB model produces superior reconstruction results compared to the PWB model.
Learned sharp representation using ICB does not match the performance of state-of-the-art deep deblurring methods on VirtualCMB, but it performs significantly better than others on RealCMB.
The variation in performance between the two datasets can be attributed to the difference in image resolution. 
The SIREN architecture utilized in the experiments may be better suited to handling low-resolution images, such as those in RealCMB.
Visual restoration examples are in Fig.~\ref{fig:deblurring_examples}. 
It is observed that the learned representation roughly restores the edges, but global noise remains in the VirtualCMB example.
On the contrary, an accurate sharp representation is obtained in the RealCMB case.
\begin{figure*}[!t]
    \centering
    \begin{tabular}{c}
         \includegraphics[width=0.94\linewidth]{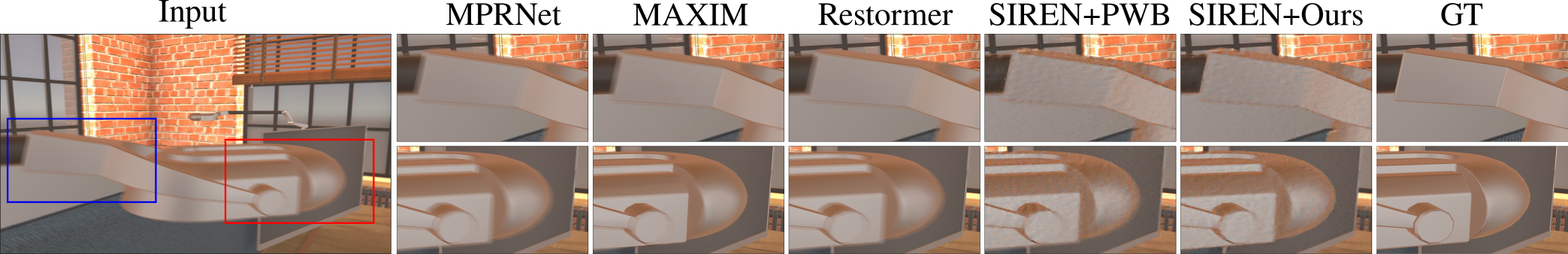} \\
         (a) \\
         \includegraphics[width=0.94\linewidth]{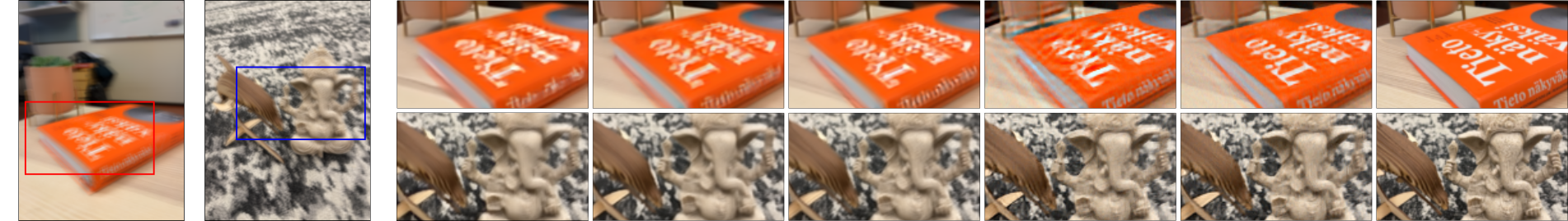} \\
         (b) \\
    \end{tabular}
    \caption{Examples of the deblurring results in (a) VirtualCMB, (b) RealCMB datasets.}
    \label{fig:deblurring_examples}
\end{figure*}
%
\section{Conclusion}
This work provides analytical and experimental results about the scene configurations in
which the scene depth affects the camera motion blur. 
In particular, we identified two types of scenes that appear in consumer photography: "Macro" and "Trucking".
Primarily, we presented an Image-Compositing Blur (ICB) model that efficiently and accurately describes the induced blur in those cases.
Experimental validation was performed in our introduced synthetic and real datasets.
Interestingly, we demonstrated the effectiveness of the ICB model to learn sharp neural representations from a single blurry image. 
Our findings and the new datasets help to develop better deblurring approaches.
\paragraph{Limitations.} 
Although our ICB model is derived for parallax motion, the model was found accurate enough under certain scene configurations, \eg,
Macro and Trucking photography.
Besides, the model is computationally efficient and robust against occlusions due to abrupt depth changes.
Regarding the deblurring task, our results are still far from being practical. 
In real scenarios, the depth maps and camera trajectories need to be estimated and that would need a careful study of the suitability of IMU-based odometry and depth sensors in the current hand-held devices. 
\clearpage
\section*{Acknowledgements}
This project was supported by a Huawei Technologies Oy (Finland) project. We also thank Jussi Kalliola for building the iOS app~\cite{chugunov_CVPR_2022} for data collection.
%
%
\bibliographystyle{splncs04}
\bibliography{refs}

\clearpage
\appendix
\pagenumbering{arabic}
\renewcommand*{\thepage}{A\arabic{page}}
\counterwithin{figure}{subsection}
\counterwithin{table}{subsection}
\counterwithin{equation}{subsection}

\section{Image Compositing Blur (ICB) model details}
\subsection{$D_l$ sequence}
In the main paper, we introduced the $D_l$ sequence that is used to determine the depth-dependent layers with similar blur behavior.
Here, we present a detailed derivation.

Eq.~\ref{eq:depth_sequence_condition} can be reorganized to obtain a recursive expression for $D_l$:
\begin{equation}
    \label{eq:recursion_D_l}
    D_l = \frac{\kappa D_{l-1}}{n D_{l-1} + \kappa} ,\hbox{ where } \kappa=\frac{s_{\max} F}{\delta} \enspace .
\end{equation}

\begin{prop}
Given $D_l$ as stated in Eq.~\ref{eq:recursion_D_l}, and $D_0 = 2\kappa$, the explicit expression for $D_l$ is:
\[ D_l = \frac{2\kappa}{2l n + 1} \]
\end{prop}

\begin{proof}
For $l=1$, we have:
\[ D_1 = \frac{\kappa D_{0}}{n D_{0} + \kappa} = \frac{2 \kappa^2}{2n\kappa + \kappa} = \frac{2 \kappa}{2n + 1} \]
It is then clear that the explicit expression holds for $l=1$.

Assuming that the statement holds for some $l'>0$, \ie $D_{l'} = \frac{2\kappa}{2l' n + 1}$, we must show that: 
\[ D_{l'+1} = \frac{2\kappa}{2(l'+1) n + 1} \]

Using Eq.~\ref{eq:recursion_D_l}:\\

   \[ D_{l'+1} = \frac{\kappa D_{l'}}{n D_{l'} + \kappa} = \frac{\kappa \big(\frac{2\kappa}{2l' n + 1}\big)}{n \big(\frac{2\kappa}{2l' n + 1}\big) + \kappa } = \frac{2\kappa }{2n + 2n l' + 1 } = \frac{2\kappa}{2(l'+1) n + 1} \]

\end{proof}

\subsection{Ablation}
Our ICB model for parallax motion includes two hyper-parameters:
\textbf{1)} $n$ controls the level of discretization in the depth map, 
the smaller the values for $n$ the more depth layers are obtained; and \textbf{2)}
$\sigma$ has an effect on the smoothness of each alpha matte $\mathcal{A}_l$.

Table~\ref{tab:ablation} reports the ablation results for 12 images within the Macro and Trucking scenes in the VirtualCMB dataset.
We can observe that $\sigma$ does not play a significant role in the performance. In contrast, $n$ balances accuracy with run time. Smaller $n$ leads to higher accuracy but in slower run time.
In consequence, we set $n=1$ and $\sigma= 4.0$ to get higher quality performance throughout our experimentation.

\begin{table}[!h]
\centering
\caption{Ablation results.}
\resizebox{0.8\linewidth}{!}{
\begin{tabular}{cc|cccc|cccc}
\toprule
  & & \multicolumn{4}{c|}{Macro}                         &  \multicolumn{4}{c}{Trucking}                           \\ 
    $\sigma$   &  $n$   & PSNR & SSIM & LPIPS ($\times 10^{-5}$) & Time [s] & PSNR & SSIM & LPIPS ($\times 10^{-5}$) & Time [s] \\ \midrule
0.5   &1.0        &40.799  &0.99142  &3.4104  &3.3310   &36.905  &0.98416  &4.7988  &2.9216  \\
      &2.0        &40.626  &0.99097  &3.2748  &2.6664   &36.725  &0.98337  &4.8637  &2.3630  \\
      &3.0        &40.393  &0.99053  &3.3637  &2.4532   &36.333  &0.98190  &6.6962  &2.1599  \\
1.0   &1.0        &40.868  &0.99151  &3.3610  &3.4100   &36.971  &0.98429  &4.7715  &2.8925  \\
      &2.0        &40.691  &0.99106  &3.2451  &2.6764   &36.786  &0.98350  &4.8065  &2.3663  \\
      &3.0        &40.463  &0.99062  &3.3218  &2.7093   &36.381  &0.98203  &6.6180  &2.2219  \\
1.5   &1.0        &40.892  &0.99154  &3.3590  &4.0521   &37.009  &0.98438  &4.7872  &2.8227  \\
      &2.0        &40.712  &0.99110  &3.2493  &2.5828   &36.824  &0.98359  &4.7764  &2.3238  \\
      &3.0        &40.492  &0.99067  &3.3150  &2.4378   &36.409  &0.98211  &6.6034  &2.1928  \\
2.0   &1.0        &40.901  &0.99156  &3.3733  &3.2586   &37.037  &0.98444  &4.8103  &2.8838  \\
      &2.0        &40.720  &0.99112  &3.2624  &2.9163   &36.854  &0.98367  &4.7623  &2.4189  \\
      &3.0        &40.507  &0.99070  &3.3235  &2.5660   &36.433  &0.98218  &6.6126  &2.5661  \\
2.5   &1.0        &40.903  &0.99157  &3.3926  &3.8899   &37.062  &0.98450  &4.8365  &2.8771  \\
      &2.0        &40.723  &0.99114  &3.2786  &2.6463   &36.881  &0.98373  &4.7593  &2.4833  \\
      &3.0        &40.515  &0.99072  &3.3375  &2.3812   &3.6456  &0.98225  &6.6283  &2.2431  \\
3.0   &1.0        &40.902  &0.99157  &3.4123  &3.4124   &37.085  &0.98455  &4.8687  &2.8062  \\
      &2.0        &40.724  &0.99115  &3.2951  &2.5933   &36.907  &0.98378  &4.7640  &2.3603  \\
      &3.0        &40.520  &0.99073  &3.3526  &2.2284   &36.477  &0.98231  &6.6487  &2.0337  \\
3.5   &1.0        &40.900  &0.99158  &3.4297  &3.7939   &37.105  &0.98459  &4.9037  &2.8090  \\
      &2.0        &40.722  &0.99116  &3.3101  &2.6797   &36.929  &0.98382  &4.7729  &2.3857  \\
      &3.0        &40.523  &0.99074  &3.3669  &2.4121   &36.495  &0.98235  &6.6714  &2.2891  \\
4.0   &1.0        &40.896  &0.99158  &3.4447  &3.3234   &37.122  &0.98462  &4.9384  &2.8267  \\
      &2.0        &40.720  &0.99116  &3.3233  &2.8096   &36.946  &0.98386  &4.7852  &2.4440  \\
      &3.0        &40.524  &0.99075  &3.3796  &2.4356   &36.509  &0.98239  &6.6945  &2.2058  \\

\bottomrule
\end{tabular}
}

\label{tab:ablation}
\end{table}

\section{Details for optimizing neural representations from blur}
The SIREN architecture \cite{sitzmann_NIPS_2020_siren} used for the neural representation consists of 4 hidden layers, each one with 192 nodes, and one out-most linear layer. We use Adam optimizer with 400 iterations. The learning rate ({\tt lr}) and the gradient weight $\lambda$ are set to $5\times 10^{-4}$ and $8\times 10^{-6}$, respectively. Likewise, we employ gradient clipping and the cosine annealing scheduler with minimum {\tt lr} of $5\times 10^{-6}$.

\section{Additional results}

More visual comparisons for our blur formation model are shown in Fig.~\ref{fig:blur_VirtualCMB}, and Fig.~\ref{fig:blur_RealCMB}. Additional deblurring results are illustrated in Fig.~\ref{fig:deblur_VirtualCMB} and Fig.~\ref{fig:deblur_RealCMB}.

\begin{figure*}[!h]
    \centering
    \begin{tabular}{c}
         \includegraphics[width=\linewidth]{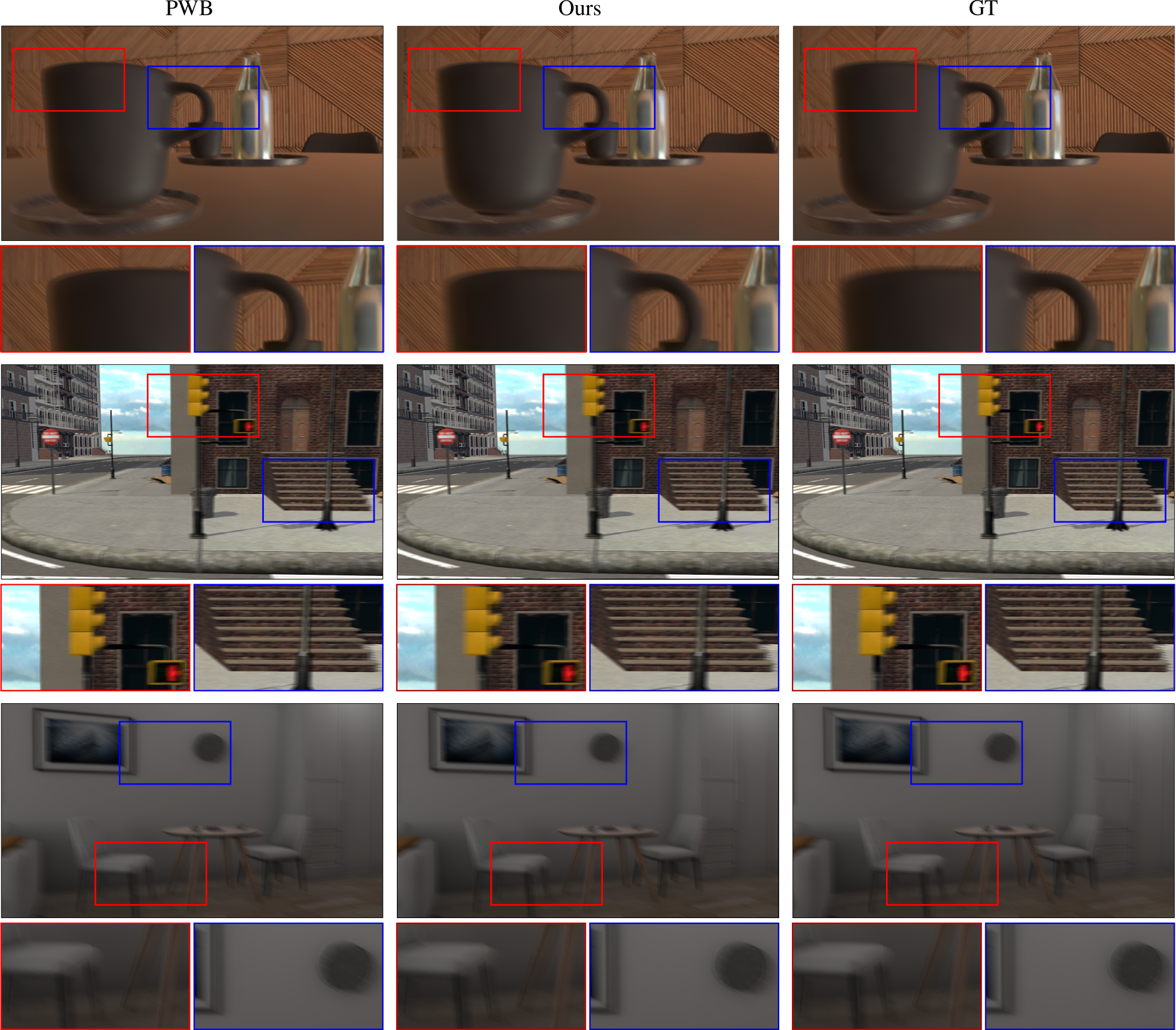} \\

    \end{tabular}
    \caption{Visual results of blur formation in VirtualCMB: Macro, Trucking, and Standard scenes.}
    \label{fig:blur_VirtualCMB}
\end{figure*}

\begin{figure*}[!h]
    \centering
    \begin{tabular}{c}
         \includegraphics[width=\linewidth]{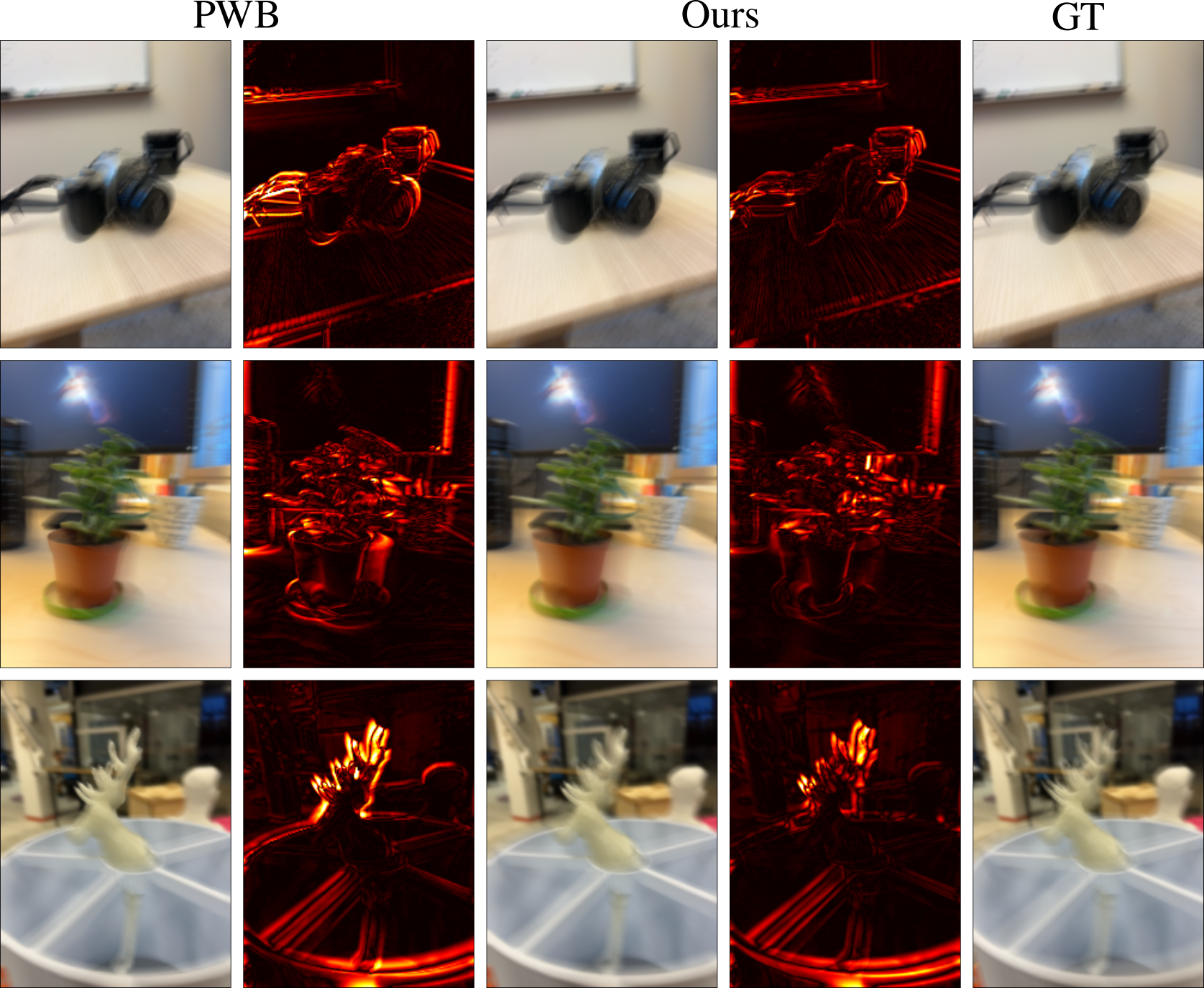} \\

    \end{tabular}
    \caption{Visual results of blur formation in RealCMB.}
    \label{fig:blur_RealCMB}
\end{figure*}
\clearpage

\begin{figure*}[!h]
    \centering
    \begin{tabular}{c}
         \includegraphics[width=\linewidth]{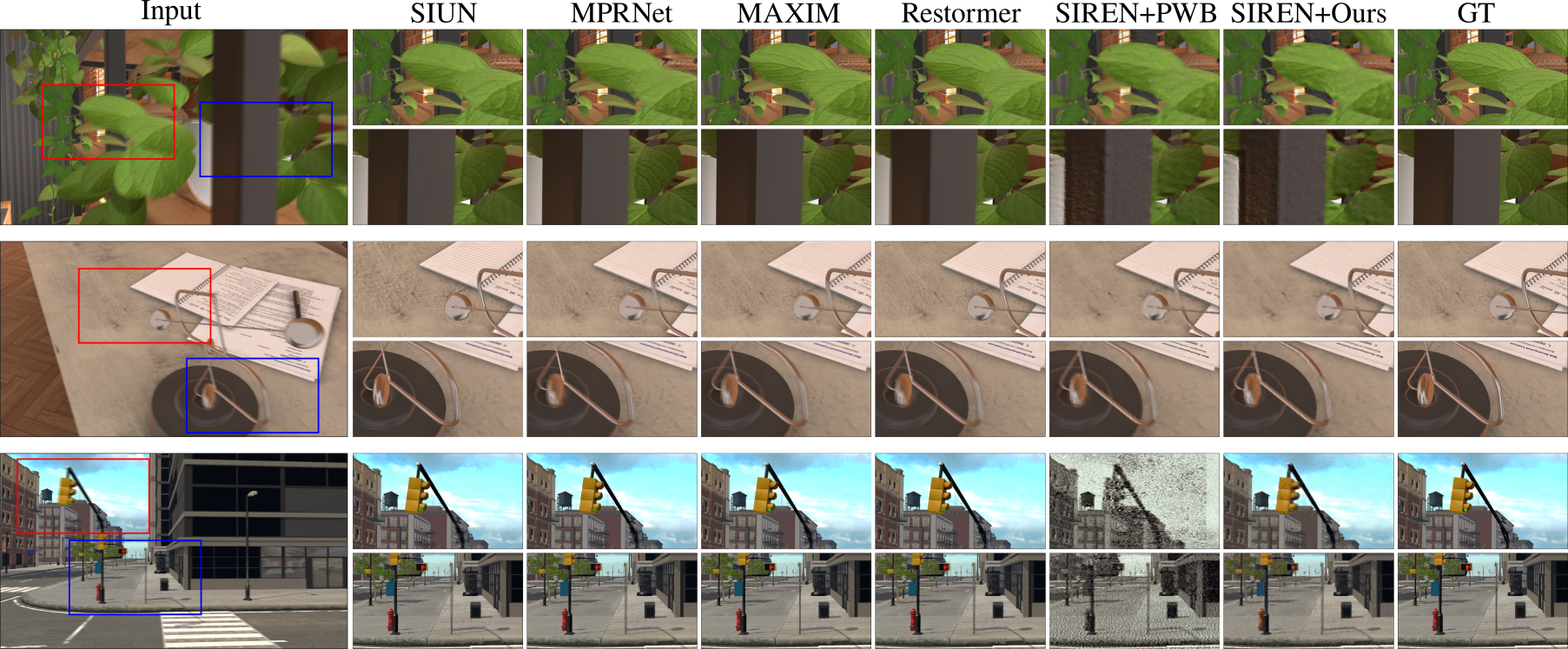} \\

    \end{tabular}
    \caption{Examples of sharp restoration in VirtualCMB.}
    \label{fig:deblur_VirtualCMB}
\end{figure*}

\begin{figure*}[!h]
    \centering
    \begin{tabular}{c}
         \includegraphics[width=\linewidth]{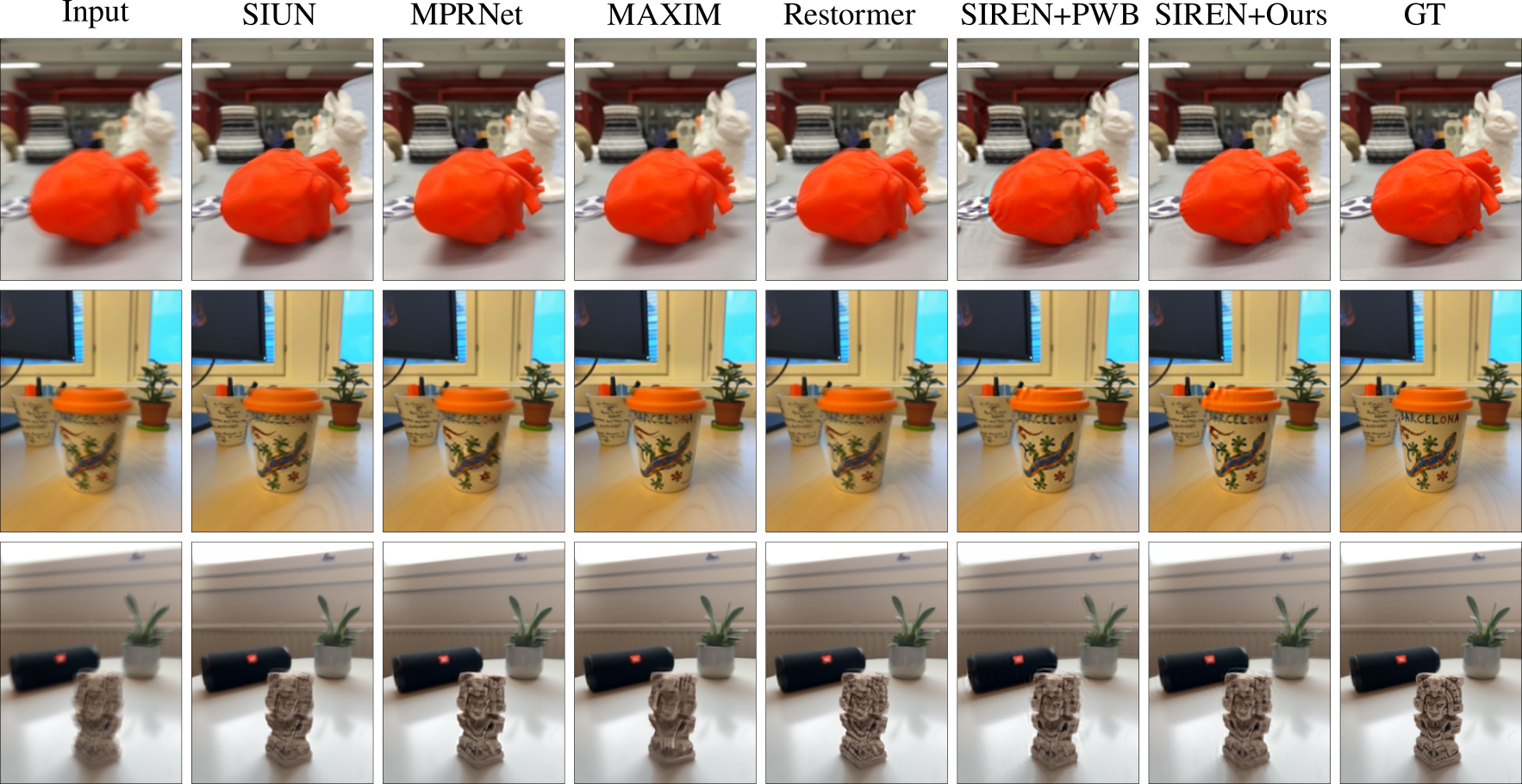} \\

    \end{tabular}
    \caption{Examples of sharp restoration in RealCMB.}
    \label{fig:deblur_RealCMB}
\end{figure*}

\end{document}